\documentclass[conference]{IEEEtran}
\IEEEoverridecommandlockouts
\usepackage{amsmath}
\usepackage{amsthm}
\usepackage{amsfonts}
\usepackage{balance}
\usepackage{graphicx}
\usepackage{algpseudocode}
\usepackage{algorithm}
\usepackage[compact]{titlesec}
\usepackage{gensymb}
\usepackage{bbm}
\usepackage{tikz}
\usepackage{textcomp}
\usepackage{hyperref}
\usepackage{lipsum}

\newtheorem{theorem}{Theorem}
\newtheorem{proposition}{Proposition}
\newtheorem{remark}{Remark}

\def\BibTeX{{\rm B\kern-.05em{\sc i\kern-.025em b}\kern-.08em
    T\kern-.1667em\lower.7ex\hbox{E}\kern-.125emX}}

\newcommand\copyrighttext{%
  \footnotesize \textcopyright 2023 IEEE. Personal use of this material is permitted.  Permission from IEEE must be obtained for all other uses, in any current or future media, including reprinting/republishing this material for advertising or promotional purposes, creating new collective works, for resale or redistribution to servers or lists, or reuse of any copyrighted component of this work in other works.}
\newcommand\copyrightnotice{%
\begin{tikzpicture}[remember picture,overlay]
\node[anchor=south,yshift=10pt] at (current page.south) {\fbox{\parbox{\dimexpr\textwidth-\fboxsep-\fboxrule\relax}{\copyrighttext}}};
\end{tikzpicture}%
}

\begin{document}

\title{Rapid and Scalable Bayesian AB Testing}

\author{
\IEEEauthorblockN{1\textsuperscript{st} Srivas Chennu}
\IEEEauthorblockA{\textit{Apple}\\
srivas.chennu@apple.com}
\and
\IEEEauthorblockN{1\textsuperscript{st} Andrew Maher}
\IEEEauthorblockA{\textit{Apple}\\
andrew\_maher@apple.com}
\and
\IEEEauthorblockN{2\textsuperscript{nd} Christian Pangerl}
\IEEEauthorblockA{\textit{Apple}\\
c\_pangerl@apple.com}
\and
\IEEEauthorblockN{3\textsuperscript{rd} Subash Prabanantham}
\IEEEauthorblockA{\textit{Apple}\\
subash@apple.com}
\and
\IEEEauthorblockN{4\textsuperscript{th} Jae Hyeon Bae}
\IEEEauthorblockA{\textit{Apple}\\
jaehyeon\_bae@apple.com}
\and
\IEEEauthorblockN{5\textsuperscript{th} Jamie Martin}
\IEEEauthorblockA{\textit{Apple*\thanks{* Contribution while at Apple.}}\\
jamiejmartin1@gmail.com}
\and
\IEEEauthorblockN{6\textsuperscript{th} Bud Goswami}
\IEEEauthorblockA{\textit{Apple}\\
b\_goswami@apple.com}
}

\maketitle
\copyrightnotice

\begin{abstract}
AB testing aids business operators with their decision making, and is considered the gold standard method for learning from data to improve digital user experiences. However, there is usually a gap between the requirements of practitioners. The constraints imposed by the statistical hypothesis testing methodologies commonly used for analysis of AB tests. These include the lack of statistical power in multivariate designs with many factors, correlations between these factors, the need of sequential testing for early stopping, and the inability to pool knowledge from past tests. Here, we propose a solution that applies hierarchical Bayesian estimation to address the above limitations. In comparison to current sequential AB testing methodology, we increase statistical power by exploiting correlations between factors, enabling sequential testing and progressive early stopping, without incurring excessive false positive risk. We also demonstrate how this methodology can be extended to enable the extraction of composite global learnings from past AB tests, to accelerate future tests. We underpin our work with a solid theoretical framework that articulates the value of hierarchical estimation. We demonstrate its utility using both numerical simulations and a large set of real-world AB tests. Together, these results highlight the practical value of our approach for statistical inference in the technology industry.
\end{abstract}


\section{Introduction}
AB testing aids business operators with their decision making, and is considered the gold standard method for learning from data to improve digital user experiences. However, there is usually a gap between the requirements of practitioners, and the constraints imposed by the statistical hypothesis testing methodologies commonly used for analysis of AB tests, including t-tests and ANOVAs. Let's take some of the most important factors in turn, and illustrate the gap that exists between requirements and common statistical constraints:

\emph{Evaluation of factors and contexts} -- many common methods in AB testing suffer from the multiple comparisons problem when aiming to understand effects of different factors or contexts on the experimental metric of success. For example, one might want to understand how language localisation to different markets impacts a new product feature, or how that feature is received differently across different devices or interfaces. This challenge is broadly framed as multivariate AB testing.

\emph{Statistical power in large-scale tests} -- as the number of categorical factors and possible values of these factors grows, the amount of traffic allocated to each combination of values reduces. Hence a t-test run separately for each such combination would suffer from reduced statistical power, and consequently require the test to be run for a longer duration before true differences can be detected.

\emph{Sequentiality and multiple comparisons} -- as the size of the test grows, t-tests and ANOVAs also incur progressively higher risk of multiple comparisons due to the large number of pairwise comparisons that can be conducted in each combination of factors. Furthermore, data in the digital services industry are typically accrued in a sequential manner. Fixed horizon methods like t-tests, ANOVAs, etc. incur progressively higher risk of false positives if used repeatedly.

\emph{Correlations between variables} -- real-world data often have correlations in them due to the nature of large-scale tests with users who share many common properties. For example, the impact of a copy test in different countries that use the same language is likely to be correlated. Conventional methods do not take this into account.

\emph{Learning across tests} -- Many organizations have a large repertoire of past AB tests. Conventional methods are unable to exploit this rich trove of data from past tests to accelerate future tests.

Here, we offer a methodology that addresses the above limitations, and provides accuracy, speed and richness of learning. We compare the performance of our methodology to a standard baseline in sequential AB testing - namely the mixture Sequential Probability Ratio Test (mSPRT) underpinned by maximum likelihood estimation, which has been deployed on industry-scale AB testing platforms \cite{Pekelis2016}.

Further, we demonstrate how this methodology can be extended to allow the sharing of composite global learnings gleaned from past AB tests.

Specifically, we develop a sequential multivariate testing framework that delivers large-scale multivariate AB testing that:

\begin{itemize}
    \item enables a large volume of statistical inference without incurring excessive false positive risk
    \item increases statistical power by exploiting correlations between experimental variables, thereby accelerating the conclusion of tests
    \item learns from historical tests to deliver faster learning in future tests
\end{itemize}

\section{Background}
Sequential multivariate AB testing can be framed as the problem of learning about the influence of multiple latent, independent variables on one or more observed, dependent variables that we are interested in. Typically, these dependent variables are business metrics that we are interested in optimizing using an AB test. Independent variables are often natural or experimentally driven variables, e.g., the different treatments in the test, the country the user is in, etc.

In the sequential batch learning setting, we receive a new batch of data about the dependent variables at regular updates. For example, we might observe a certain number of clicks or downloads by users who were shown a certain message. Using these data, we update our knowledge of the contribution of the independent variables to the dependent variable. Given this updated knowledge, sequential AB testing involves making inferences about statistically significant differences between the variables: we might want to infer whether or not a certain message performs significantly better than another in a particular country.

To simplify the presentation here, we restrict ourselves to a single, dependent variable influenced by multiple independent variables, though our approach
can be extended to multiple dependent variables. To use a practical example, consider an AB test comparing messages sent to users, each message consisting of a title, an image and body text. The probability of the user responding to the message (the dependent variable) can be modelled as being influenced by multiple independent variables, including properties of the 
content they are shown, e.g., the image, the title and the body, as well as the user's context, e.g., the country they live in, the type of device they are using, etc.

\subsection{Multivariate AB testing}\label{subsec:multivariate_ab}
Standard univariate AB testing would involve testing each content element (image, body or title) at a time, separately in each
country. While this simplifies the consequent statistical analysis, it ignores potential interactions between the image, title and body, and between this content and the user's context. In contrast, multivariate AB testing addresses this limitation by jointly modelling the influence of such content and context factors on the user's response.

The design of such a multivariate AB test requires specification of the content and context factors that influence the probability of a user's response. Again, for the sake of simplicity, we restrict ourselves to categorical factors. For example, the image presented to the user is a categorical content factor that can take one value amongst a set of values, each of which corresponds to a particular image in a fixed set. Similarly, the user's country or device type is a categorical context factor.

Specifically, we have a total of $M$ content and $C$ context factors. A given content factor with index $i$ assumes one-hot encoded representations of categorical values in the set $\mathcal{M}_i$, $i=1,\dots,M$, while a context factor indexed by $j$ assumes one-hot representations of categorical values in $\mathcal{C}_j$, $j=1,\dots,C$. All $\mathcal{M}_i$ and $\mathcal{C}_j$ have a cardinality of at most $N$. Furthermore, $\mathcal{M} = \mathcal{M}_1 \times \cdots \times \mathcal{M}_M$ denotes the set of all content combinations and $\mathcal{C} = \mathcal{C}_1 \times \cdots \times \mathcal{C}_C$ represents the set of all context combinations, respectively. 

Given this formulation, the true but unknown probability of a response is represented as $r_{f}$, where the vector $f = (m,c)$ is an element of $\mathcal{F} = \mathcal{M} \times \mathcal{C}$ and denotes a specific content $m$ presented to the user with a particular context $c$. Given a total of $F = M + C$ content and context factors, each with cardinality of at most $N$, the number of unique $f$ grows exponentially as $O(N^F)$. One of the statistical challenges with multivariate AB testing is that, as $N$ and $F$ grow, it becomes increasingly challenging to estimate each $r_{f}$ with sufficient statistical power. Our hierarchical Bayesian approach ameliorates this challenge, by pooling knowledge across different instances of $f$ to increase statistical power.

\subsection{Sequential hypothesis testing}
Suppose we have some estimate $\hat{r}^{h}_f$ of the true $r_{f}$, where the superscript $h$ denotes the use of a hierarchical Bayesian estimate. The next challenge is to define a robust method to enable experimenters to test hypotheses about statistically significant differences between relevant pairs of estimates. In sequential multivariate testing, experimenters need the ability to evaluate evidence for these hypotheses at each sequential update, in order to stop the test early and to make a roll-out decision without sacrificing statistical validity. This can become a challenge when applying conventional statistical methods for multivariate AB testing at scale, due to the large number of potential multiple comparisons that are possible between pairs of maximum likelihood estimates of $\hat{r}_f$. In practice, this risk is reduced with post hoc error correction methods that adjust p-values to limit the false positive or false discovery rate, e.g., the Bonferroni method \cite{bonferroni1936teoria} or the Benjamini-Hochberg method \cite{Benjamini1995ControllingTesting}.

Here, we use a Bayesian hypothesis testing framework that builds on the mSPRT \cite{Johari2015AlwaysTesting}. The framework sequentially evaluates the relative evidence that there is a statistically significant difference between pairs of Bayesian estimates $\hat{r}^{h}_f$s -- all the while maintaining statistical validity without the need for post hoc corrections.

\subsection{Learning effect-size meta-priors}

A further important challenge comes from trying to expand learnings beyond the remit of a single experiment. Large experimentation platforms can launch hundreds, thousands or even tens of thousands of experiments. From one perspective, each experiment is a distinct entity –- it might, perhaps, try to improve the conversion rate on a certain website or to boost the revenue generation elsewhere. Another perspective is that experiments can be characterised by a set of common features. Within this perspective, pooling information and learnings across experiments can help experimenters build robust intuition as to the sorts of features that yield most impactful experiments.

This pooling is referred to as experimental meta-analysis, and is well-established within psychology and medicine. On large-scale experimentation platforms in the digital services industry it is possible to go a step further. Not only can we build soft intuition related to impactful experiments, we can also quantify and operationalise this knowledge via so-called meta-prior learning or transfer learning \cite{Deng2016ContinuousTesting, Deng2015ObjectiveExperiments}. For example, if historical experiments suggest that headlines generate large impacts on clickthrough rates on an article, but that changes to body text have relatively little effect, this information can be directly incorporated into future experiment designs.

Hierarchical Bayesian inference, in addition to being useful for modelling the variables \emph{within} an experiment, can also be used to conduct such meta-analysis \emph{across} experiments. By encoding distributional assumptions, we can learn the latent hyperparameters that can explain the impact of different experimental interventions. We can also plug these learnt parameters back into an automated hypothesis testing framework to enable more refined and accelerated decision-making that benefits from knowledge built up over past experiments. Our third contribution is exactly this: we take a large suite of historical real-world experiments and demonstrate the value of this learning for enhancing future experiments.

\section{Related Work}
Hierarchical Bayesian inference is a well-established methodology previously articulated by Gelman and others \cite{gelman2004bayesian, Gelman2012WhyComparisons}. In the industry context, previous work has popularised the idea of using Bayesian inference in digital experimentation, for reducing estimation error with shrinkage \cite{Chennu2021SmoothFeedback, Dimmery2019ShrinkageExperiments, Xu2022EmpiricalExperiments}. Sequential statistical inference using Bayes factors \cite{Schonbrodt2017SequentialDifferences, Lindon2022Anytime-ValidAdjustment} has seen a recent upsurge in interest. The general methodology of Bayesian linear modelling has been successfully applied to realise contextual optimisation using multi-armed bandits \cite{Hill2017AnOptimization, Nabi2022BayesianBayes}. Further, there is a rich history of applying Bayesian hypothesis testing after specifying a suitable distribution model for the data accrued during an experiment \cite{Kass1993BayesPractice}. Our work on learning meta-priors for effect sizes across experiments relates to existing work in this space \cite{Deng2016ContinuousTesting, Deng2015ObjectiveExperiments}. We build upon these ideas to propose a common, cohesive framework that can be used to learn both within and across experiments. In particular, our approach extends \cite{Hill2017AnOptimization} with a hierarchical Bayesian framework to support robust sequential, multivariate hypothesis testing. Further, we demonstrate that this hierarchical framework can be extended to learn effects across experiments as well. In doing so, we exploit the scale of modern digital experimentation to achieve greater speed and statistical robustness.

\section{Contribution}
Our contribution combines the following key ideas that employ Bayesian inference to enable rapid, robust large-scale multivariate AB testing:

\begin{itemize}
\item hierarchical Bayesian inference for sequential estimation of user response probabilities modelled as multivariate distributions
\item Bayesian hypothesis testing to sequentially evaluate hypotheses comparing multivariate response probabilities
\item hierarchical Bayesian inference of priors for treatment effect sizes from past tests
\end{itemize}

While these ideas have been described in the previous literature highlighted above, we describe a cohesive integration of these ideas applied in practice. Using real-world examples, we comparatively evaluate our solution and demonstrate that it enables practical multivariate AB testing that is of interest to a broad cross-section of the AB testing community.

\section{Methodology}
We describe our method in three parts -- hierarchical Bayesian inference for estimating response probabilities, sequential hypothesis testing using Bayes factors, and hierarchical Bayesian inference for learning effect size meta-priors.

\subsection{Hierarchical Bayesian inference} \label{subsec:HB_inference}

Our method extends the classical generalised linear model (GLM) and introduces a hierarchical Bayesian prior on the relationship between the dependent and independent variables. In an AB test that includes the user's country as a context factor, we treat the contribution of this factor as a random variable with a prior. We model these priors themselves as random draws from a common meta-prior representing the overall contribution of all countries pooled together. Together, these priors constitute a hierarchy of distributions. At every sequential update of new data, all the priors in the hierarchical model are jointly updated using Bayesian inference, so that we can estimate posterior distributions at each level in the hierarchy. To do this efficiently and quickly at scale, we use parallelised Markov Chain Monte Carlo (MCMC) estimation implemented in numpyro \cite{phan2019composable,bingham2019pyro} and JAX \cite{jax2018github}.

More formally, the probability of a user response $\hat{r}^{h}_f$ given a factor vector $f$ is modelled as:
\begin{align}
    \hat{r}^{h}_f = g(X\beta + \epsilon),
\end{align}

\noindent where $X$ is the design matrix derived from the experimental specification, $\beta$s are coefficients modelling the contributions of the experimental factors represented as Gaussian distributions, and $\epsilon \sim \text{Normal}(0, 1)$ represents zero-mean noise.

The one-hot encoded design matrix $X$ specifies the contribution of these $\beta$s to each $\hat{r}^{h}_f$. The number of rows of $X$ grows as $O(N^F)$, corresponding to each possible value of $f$. However, the number of columns of $X$ grows as $O(NF)$, where each column represents the contribution of a particular value $v$ of a factor vector $f$. Specifically, each row $X_k$ of $X$ corresponds to a factor vector $f = (m, c)$ with $m = (m_1,\dots,m_M)$ and $c = (c_1, \dots, c_C)$ where $m_i$ and $c_j$ are one-hot encoded members of the above introduced sets $\mathcal{M}_i$ and $\mathcal{C}_j$, respectively. Hence, by slightly abusing notation and assuming that the vector $f$ is the concatenation of its one-hot encoded components, $X_k$ can be written as

\begin{align} \label{eq:X}
    X_{k} = f, \quad 
    f\in \mathcal{F}.
\end{align}

The above design matrix can be extended to support non-linear relationships by including $\beta$s corresponding to interactions between factors. For example, modelling interactions between pairs of factor values $v_1$ and $v_2$ would result in $X$ having $O(N^{2}F^{2})$ columns.

Using a hierarchical Bayesian formulation, the $\beta$s for each experimental factor are sampled from the following generative model:

\begin{align}
\beta \sim \text{Normal}(\mu,\,\sigma^{2}), \label{eq:beta_full_model}\\
\mu \sim \text{Normal}(0, 100), \label{eq:mu_full_model} \\
\sigma \sim \text{HalfCauchy}(5)\label{eq:sigma_full_model}.
\end{align}

In the above model, $\mu$ and $\sigma$ represent the mean and standard deviation
of the prior distribution from which individual $\beta$s are sampled.

When we fit the above model, we estimate the Bayesian posterior distribution of each $\beta$,
but also that of $\mu$, $\sigma$ and $\hat{r}^{h}_f$. As we are modelling $\hat{r}^{h}_f$ as probabilities, $g$ is the sigmoid function.

This model is fitted to binomially distributed counts of
treatment assignments $a$ and responses $r$, which relate to $\hat{r}^{h}_f$ as:

\begin{equation}
    r = \text{Binomial}(a, \hat{r}^{h}_f).
\end{equation}

We refer readers to Appendix \ref{app:ps_code_BHM} for a more detailed discussion of implementation details. We have also included pseudocode to facilitate reproducibility of our work.

\subsubsection{Marginal probability estimation}
\label{marginal_probability_estimation}

We use the $\hat{r}^{h}_f$s estimated above to calculate response probability distributions per content factor combination, marginalising over context factor combinations. This enables experimenters to gain global insights about the performance of the content presented to users, marginalising over contextual factors.

These marginal response probability distributions $r_m$ are computed as weighted averages across the context factors $C$ as:

\begin{equation}
    \hat{r}^{h}_m = \sum_{c \in C} w_{c} \; \hat{r}^{h}_{f}\;, \quad
    w_{c} \propto t_{c}
\end{equation}

where $t_c$ is the number of treatment assignments per context combination $c$.

\subsubsection{Comparison with maximum likelihood estimation}
We compare our hierarchical Bayesian approach to a popular baseline - the maximum likelihood estimation (MLE) method, which underpins both fixed horizon AB testing methods like the t-test and sequential AB testing methods like the mSPRT. Here, the maximum likelihood estimate of $r_f$ and the variance of this estimate would, respectively, be:

\begin{align*}
\hat{r}^{l}_f &= \frac{r}{a}, \\
\hat{v}^{l}_f &= \frac{\hat{r}^{l}_f(1 - \hat{r}^{l}_f)}{a}.
\end{align*}

Maximum likelihood estimation is commonly used in frequentist statistical inference, e.g., in the t-test. From a statistical inference perspective, a key advantage of hierarchical estimation is that it shares knowledge across $\hat{r}^{h}_f$s via the hierarchical prior, and hence reduces variance faster than maximum likelihood estimation. From a machine learning perspective, the hierarchical prior acts as regularisation term, increasing robustness of the model by reducing the risk of overfitting.

In Appendix \ref{app:A} we prove that, for a simplified hierarchical model and under certain conditions, this partial pooling mechanism does indeed lead to faster variance reduction of the Bayesian estimator as compared to the maximum likelihood estimator. As described in Sec. \ref{evaluation}, we complement this theoretical result with detailed numerical simulations and evidence from real-world experiments that demonstrate the value of the hierarchical Bayesian approach.

\subsection{Statistical hypothesis testing}
\label{statistical_hypothesis_testing}

We adopt a hypothesis testing framework that uses Bayes factors for sequentially evaluating relative evidence for statistically significant differences between a pair of estimates of $r_f$, given by $\hat{r}_{f,A}$ and $\hat{r}_{f,B}$, estimated using the hierarchical Bayesian or the maximum likelihood method. We entertain a pair of hypotheses: a null hypothesis $\text{H}_0$ that the $\hat{r}_f$s are statistically equivalent, and an alternative hypothesis $\text{H}_1$ that they are significantly different. We frame these as:

\begin{align}
\text{H}_0&:\quad\hat{r}_{f,A} - \hat{r}_{f,B} = 0, \\
\text{H}_1&:\quad\hat{r}_{f,A} - \hat{r}_{f,B} \neq 0.
\end{align}

Assuming that $\text{H}_0$ and $\text{H}_1$ are equally likely a priori, we construct prior distributions of the true mean difference between the $\hat{r}_f$s under both hypotheses. $\text{H}_0$ is framed simply as a point at zero, and $\text{H}_1 \sim \mathrm{Normal}(0, \tau)$, i.e., a zero mean normal distribution with variance $\tau$. The value of $\tau$, which effectively represents the variability in true effect sizes in AB tests, is set in one of three ways:

\begin{itemize}
    \item \emph{fixed} -- $\tau$ is set to a fixed value
    \item \emph{dynamic} -- $\tau$ is set to the squared observed difference between the pair of $\hat{r}_f$s, i.e., $\tau = (\hat{r}_{f,A} - \hat{r}_{f,B})^2$ \cite{Zhao2018SafelyDesign}
    \item \emph{learnt} -- $\tau$ is set to a value learnt from observed effect sizes in past tests (see Sec. \ref{effect_size_prior_learning})   
\end{itemize}

Given these priors, we calculate the Bayes factor \cite{Kass1995BayesFactors} as a likelihood ratio. Specifically, we compute the relative likelihood
of $\hat{r}_{f,A} - \hat{r}_{f,B}$ under the priors corresponding to the two
hypotheses:

\begin{equation}
    \mathrm{K}_{\hat{r}_{f,A},\hat{r}_{f,B}} = \frac{p(\hat{r}_{f,A} - \hat{r}_{f,B} \; | \; \text{H}_1)}
    {p(\hat{r}_{f,A} - \hat{r}_{f,B} \; | \; \text{H}_0)}.
\end{equation}

Bayes factors quantify the relative evidence for the two hypotheses. Values near 1 indicate absence of evidence for any difference between the treatments, whereas larger
values suggest evidence that a difference exists. Bayes factors are interpreted as statistical confidence by inverting them to sequential p-values as:

\begin{equation} \label{eq:seq_p_value}
    p = \frac{1}{\mathrm{K}_{\hat{r}_{f,A},\hat{r}_{f,B}}}.
\end{equation}

\noindent The smallest p-value observed over updates is retained and reported for statistical interpretation and decision making.

\subsubsection{Multiple comparisons}
We estimate Bayes factors comparing $\hat{r}_f$s of each pair of content values. For example, in a multivariate AB test where users in 4 countries are sent a message containing one of 2 titles and 2 images, which they might view on one of 4 device types, there are 16 ($4^2$) combinations of country and device type contexts, and 4 ($2^2$) possible combinations of title and image content. Consequently, we calculate Bayes factors for a total of $16 * 6 = 96$ pairwise comparisons, corresponding to ${4 \choose 2} = 6$ comparisons per combination of context values.

As the number of such multiple comparisons grows very quickly with large experimental designs, conventional statistical methods for multivariate statistical analysis require post hoc correction of p-values to manage the risk of generating false positives.

We demonstrate that our hierarchical Bayesian estimation approach reduces the risk of false positives without the need for post hoc correction. This is because the hierarchical model ``shrinks'' the $\hat{r}_f$s towards each other as a function of their respective variances \cite{Gelman2012WhyComparisons}, thereby reducing spurious differences between them.

\subsection{Effect-size prior learning}
\label{effect_size_prior_learning}

The same hierarchical Bayesian inference framework can be applied to the problem of effect size prior learning. In contrast to common meta-learning frameworks, we care much less about the unobserved, latent average effect size. Indeed, we assume that, in aggregate, experiments cancel each other out, and that the average effect size is zero. Instead, we focus on learning the variance in effect size – i.e., what is the dispersion among experiments?

We assume that the observed effect size in an experiment is a random draw from a distribution with unknown parameters, and use hierarchical Bayesian estimation to learn these parameters. Formally, we model the experiment-level effect size $\delta_i$ as a Normal distribution:

\begin{align}
    \delta_i \sim \text{Normal}(0, \sigma_i^2 + \tau), \\
    \tau \sim \text{HalfCauchy}(5).
\end{align}

As highlighted, in our model we enforce the distribution to be zero-centred, and the impact on dispersion is governed by the experiment-level signal-to-noise ratio as well as a global, learnable, dispersion parameter $\tau$.

\subsubsection{Incorporating meta-priors into experimentation}

The learnt meta-priors are useful for informing better experiments; the $\tau$  parameter summarises the overall utility of the experimentation platform in detecting statistical effects. Further, it can be directly operationalised, to enhance the sequential hypothesis testing methodology. As previously noted, the mSPRT contains a Goldilocks parameter that must be set correctly to achieve well-powered experiments. This parameter, $\tau$ is exactly what we learn in the above formulation. We can input the learnt value directly back into future experiments for faster learning on the same volume of data.

\section{Evaluation}
\label{evaluation}

We evaluate the performance of the hierarchical Bayesian approach using both large-scale numerical simulations and results from real-world AB test data. We also compare this performance to the commonly used baseline in AB testing - the maximum likelihood method that underpins both the t-test and the mSPRT.

\subsection{Simulation framework}

\begin{figure}[hbt!]
  \centering
  \includegraphics[width=\linewidth]{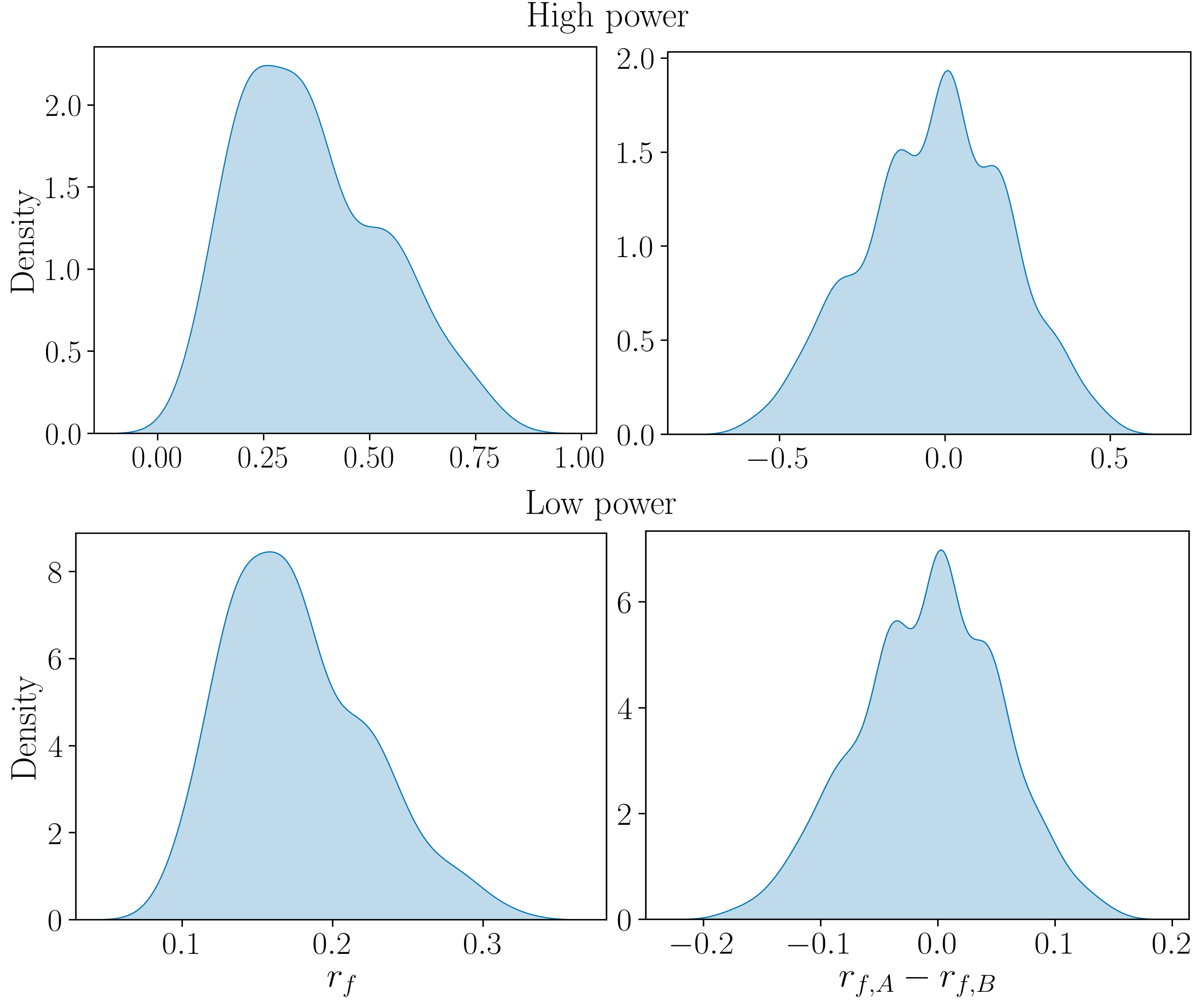}
  \caption{True $r_f$s and pairwise differences under $\text{H}_1$ simulations with high and low statistical power.}
  \label{true_sim}
\end{figure}

We simulate sequential AB tests of large multivariate experimental designs containing multiple context and content factors and interactions between them, resulting in a large number of $r_f$s being estimated at each update. Specifically, we include two context factors and two content factors, with four values each, resulting in 16 combinations of context values and 16 combinations of content values. Hence we estimate 256 ($= 16 \times 16 $) $\hat{r}_f$s at each sequential update. In addition, within each of the 16 combinations of context values, we compare $\hat{r}_f$s corresponding to each of ${16 \choose 2} = 120$ pairs of combinations of content values. In total, this results in $1920 = 120 * 16$ pairwise comparisons across all combinations of context values being conducted at each update.

When modelled with the hierarchical Bayesian approach, this experimental design is represented by 16 ($ = 4 * 4$) first-order $\beta$s that model the contribution of each value of each factor. In addition, we include 96 second-order $\beta$s that model the interactions between each pair of values of each pair of factors. Finally, we also include a single intercept $\beta$ representing the common mean of the contributions of all the $\beta$s. This results in a binary design matrix $X$ that has 256 rows and 113 ($ = 16 + 96 + 1$) columns, containing a 1 where the $\beta$ contributes to the $r_f$, and 0 otherwise.

To generate true values of $r_f$s under $\text{H}_1$, we set the first-order coefficients to zero, but randomly sample the second-order interaction coefficients for half of the $r_f$s from a normal distribution. We compare the hierarchical Bayesian and maximum likelihood estimation methods under two distinct scenarios:

\begin{itemize}
\item \emph{high statistical power} -- where the interaction coefficients are sampled from a normal distribution with mean and standard deviation 0.5, and we make a total of 100k treatment assignments per sequential update
\item \emph{low statistical power} -- where the interaction coefficients are sampled from a normal distribution with mean and standard deviation 0.2, and we make a total of 2.5k treatment assignments per sequential update
\end{itemize}

Fig. \ref{true_sim} depicts the distribution of the resulting $r_f$s and the true pairwise differences between them, in the above two scenarios. In each scenario, we also simulate $\text{H}_0$ by setting the interaction coefficients for the second half of the $r_f$s to zero, resulting in identical $r_f$s with no real differences between them.

We use these true $r_f$ values to generate binomially distributed treatment assignments and responses over 30 sequential updates. We distribute the overall number of treatment assignments per update (depending on the scenario) across the $\hat{r}_f$s to be estimated by our learning model, hence simulating an equal allocation sequential AB test. We conduct 80 random repetitions of the simulated AB test and present results that average across these repetitions. At each update, we record the 256 $\hat{r}_f$s estimated, as well as the 1920 Bayes factors $\mathrm{K}_{\hat{r}^h_{f,A},\hat{r}^h_{f,B}}$ from the pairwise comparison of $\hat{r}_f$s. We convert the Bayes factors to sequential p-values as described in ~\eqref{eq:seq_p_value} above.

As noted in sec. \ref{statistical_hypothesis_testing}, estimating Bayes factors and sequential p-values requires the specification of the $\tau$ parameter of the mSPRT. We simulated three possibilities for $\tau$: a fixed value of $0.1$, the dynamic setting proposed by Zhao et al. \cite{Zhao2018SafelyDesign}, as well as setting it to a value learnt from meta-analysis of previous simulations. We do this by constructing an effect-size meta-prior distribution as described earlier. We separate the 80 random repetitions into two equally-sized groups. The first group represents a training set, from which we learn $\tau$. The latter group is used for testing the learnt value.

Given this simulation setup, we measure the performance of the hierarchical Bayesian estimation relative to the maximum likelihood estimation using metrics of accuracy of estimation, as well as accuracy of hypothesis testing.

\subsubsection{Estimation accuracy}
We measure estimation accuracy of $\hat{r}^{h}_f$ and $\hat{r}^{l}_f$ using root mean squared error:

\begin{equation}
\text{RMSE} = \sqrt{(r_f - \hat{r}^k_{f})^2}, \quad k \in \{h, l\}.
\end{equation}

\subsubsection{Decision accuracy}
We define a conventional 5\% threshold on the level of significance required to declare a statistically significant difference. We declare a significant difference if the sequential p-value crosses this threshold, i.e., $p < 0.05$, at any update during a simulation repetition, and measure the sequential false negative rate under $\text{H}_1$, the sequential false positive rate under $\text{H}_0$, and the overall false discovery rate, averaging across repetitions.

\subsection{Real-world performance}

To complement the simulations, we also evaluated the impact of applying these methodologies to an anonymised, aggregated data set generated at our company. This data set comprised 220 experiments, each of which varied an aspect of user experience, aiming to optimise a rate metric estimated as a ratio of two other metrics. These two metrics were recorded at regular sequential updates and used as input to the learning model. Across experiments in the data set, there were up to 60 sequential updates, with an average of 17. Within each experiment, impressions were associated with context and content factors that defined limited, technical aspects of the environment in which it was displayed. The average experiment comprised 22 combinations of context and content factor values, and the largest 117.

As with the simulations, we use the hierarchical Bayesian framework to estimate the distributions of $r_f$s at each sequential update, for each combination of context and content factor value. We measure Bayes factors and sequential p-values for all pairs of $\hat{r}_f$s within each combination of contextual factor values in experiments in the test set. We then use the 5\% level of significance threshold on the sequential p-value to decide whether there are significant differences between the $\hat{r}_f$s.

When conducting hypothesis tests, we either use a fixed $\tau$ of 0.1, or learn it from the data set itself. To do so, we separate the experiments into two distinct groups: we take the first half of experiments as a training set to learn $\tau$, and test this learnt value on the latter half of experiments.

\section{Results}

\subsection{Estimation performance}

\begin{figure}[hbt!]
  \centering
  \includegraphics[width=\linewidth]{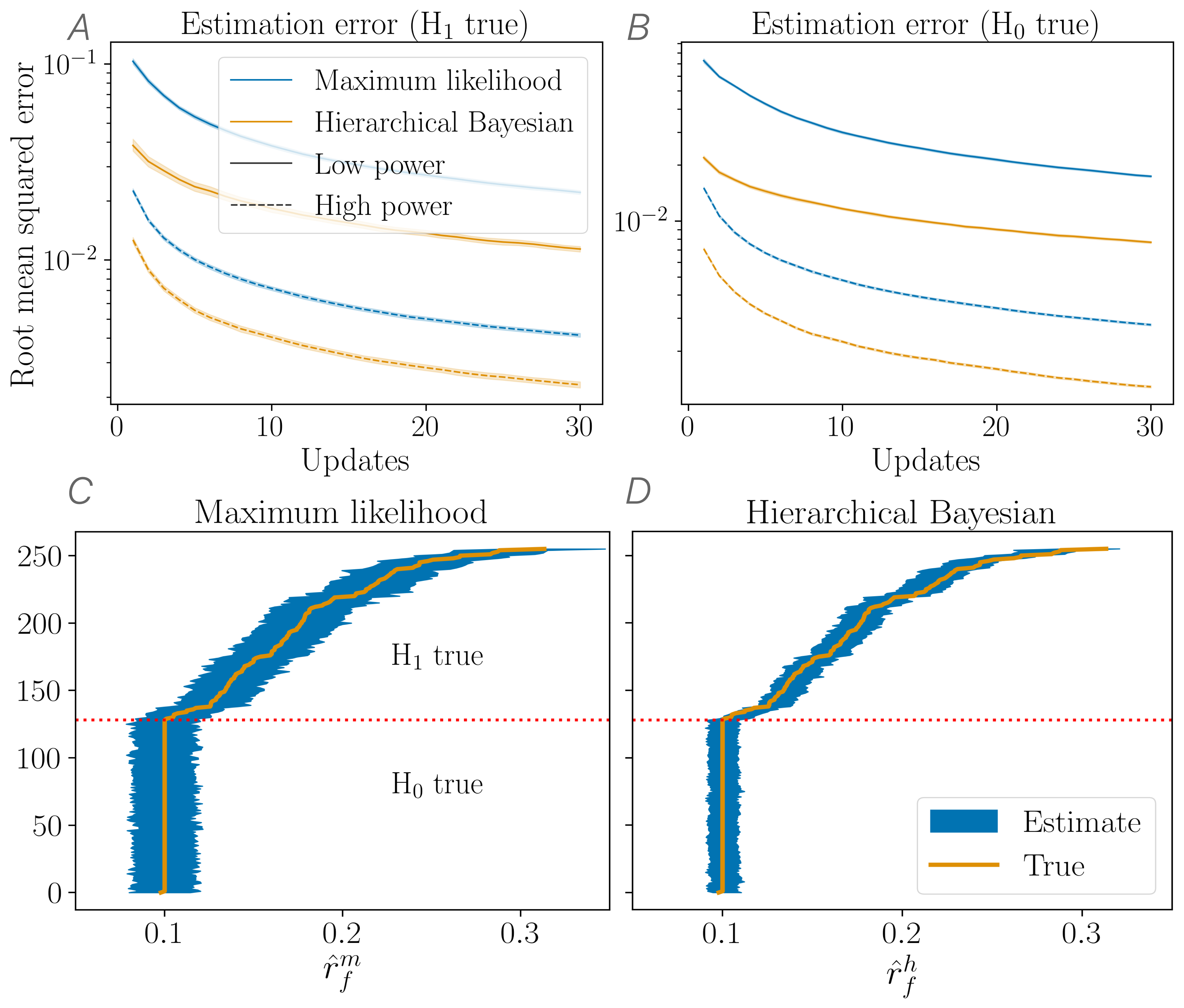}
  \caption{Hierarchical Bayesian vs. maximum likelihood
  estimation of $r_f$. Panels A and B plot estimation error. Panels C and D plot estimates at the last sequential update in the low power scenario, sorted by the true $r_f$, and averaged over repetitions.}
  \label{estimation_error}
\end{figure}

Figs. \ref{estimation_error}A and \ref{estimation_error}B compare the estimation error of maximum likelihood ($\hat{r}^{l}_f$) vs. hierarchical Bayesian ($\hat{r}^{h}_f$) estimation. In both effect size scenarios, the latter method produced a faster reduction in estimation error over sequential updates. This was true both when there were real differences between the true $r_f$s ($\text{H}_1$ true) and when the true $r_f$s were all identical ($\text{H}_0$ true).

Figs. \ref{estimation_error}C and \ref{estimation_error}D compare the variance of the estimators, demonstrating that the hierarchical Bayesian estimator has lower variance under both hypotheses and statistical power scenarios. This emerges due to the pooling of information across the levels in the hierarchical model. It is worth noting that this benefit grows with the size of the experiment, i.e., the greater the number of $r_f$s being estimated in a multivariate test, the greater the advantage of the hierarchical method.

\subsection{Sequential hypothesis testing}

\begin{figure}[hbt!]
  \centering
  \includegraphics[width=\linewidth]{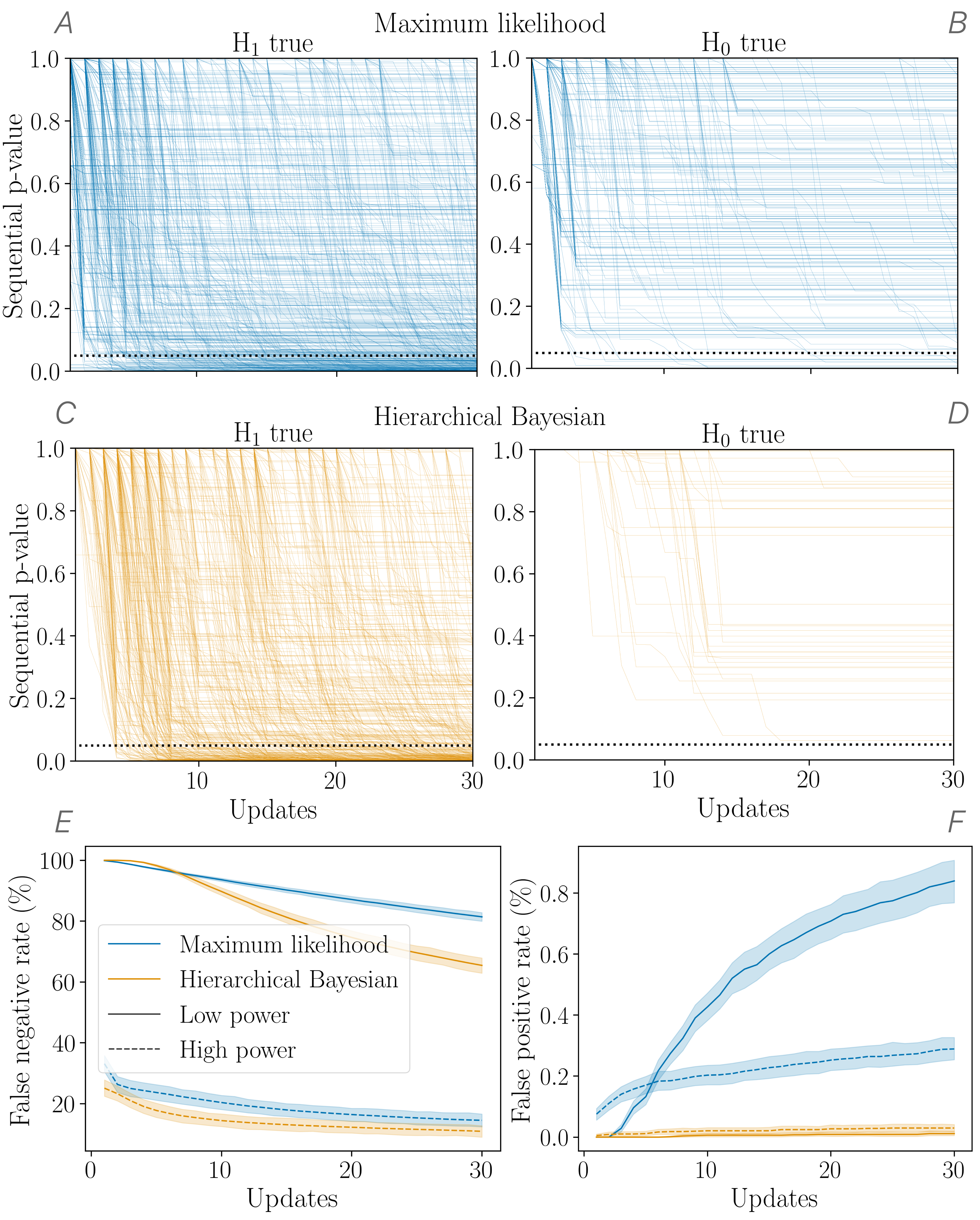}
  \caption{Panels A-D compare sequential p-values resulting from hierarchical Bayesian vs. maximum likelihood estimation, in one simulation repetition. Panels E-F plot overall false negative and positive rates across simulations in low and high power scenarios.}
  \label{statistical_confidence}
\end{figure}

We measured sequential p-values using a fixed $\tau$ of $0.1$. Figs. \ref{statistical_confidence}A-D  depict 1920 sequential p-value traces from one of the 80 repetitions, derived from Bayes factors comparing pairs of $\hat{r}_f$s in the low power scenario. As compared to the maximum likelihood method (figs. \ref{statistical_confidence}A-B), the hierarchical Bayesian method shows a visually stronger distinction in the sequential evolution of p-values between $\text{H}_0$ and $\text{H}_1$ (figs. \ref{statistical_confidence}C-D). In turn, this means that the hierarchical Bayesian approach produces a lower false negative rate under $\text{H}_1$ (fig. \ref{statistical_confidence}E) and a lower false positive rate under $\text{H}_0$ (fig. \ref{statistical_confidence}F). This improvement relative to maximum likelihood is evident under both low and high statistical power scenarios. However, it is worth noting that, under conditions of low power, the false negative rate initially takes a bit longer to reduce with the hierarchical Bayesian method, but then reduce more rapidly afterwards (see Fig. \ref{statistical_confidence}E). In practice, this confers a useful benefit for controlling error rates under conditions of low signal-to-noise.

Finally, we highlight that these sequential methods achieve a far lower false positive rate than a t-test, if it were to be used sequentially to assess statistical significance. More specifically, instead of using sequential Bayes factors, if we were to use a t-test to compare a pair of identical treatments (each with a 50\% conversion rate) at every sequential update, the false positive rate grows to ~28\% at a 5\% level of significance by the 30th update. The fact that classical fixed-horizon testing approaches like t-tests, if evaluated sequentially, lead to inflated false positive rates has also been highlighted by Johari et al. (see fig. 2 in \cite{Johari2017PeekingIt}).

\subsection{Meta effect-size prior learning}

\begin{figure}[hbt!]
  \centering
  \includegraphics[width=\linewidth]{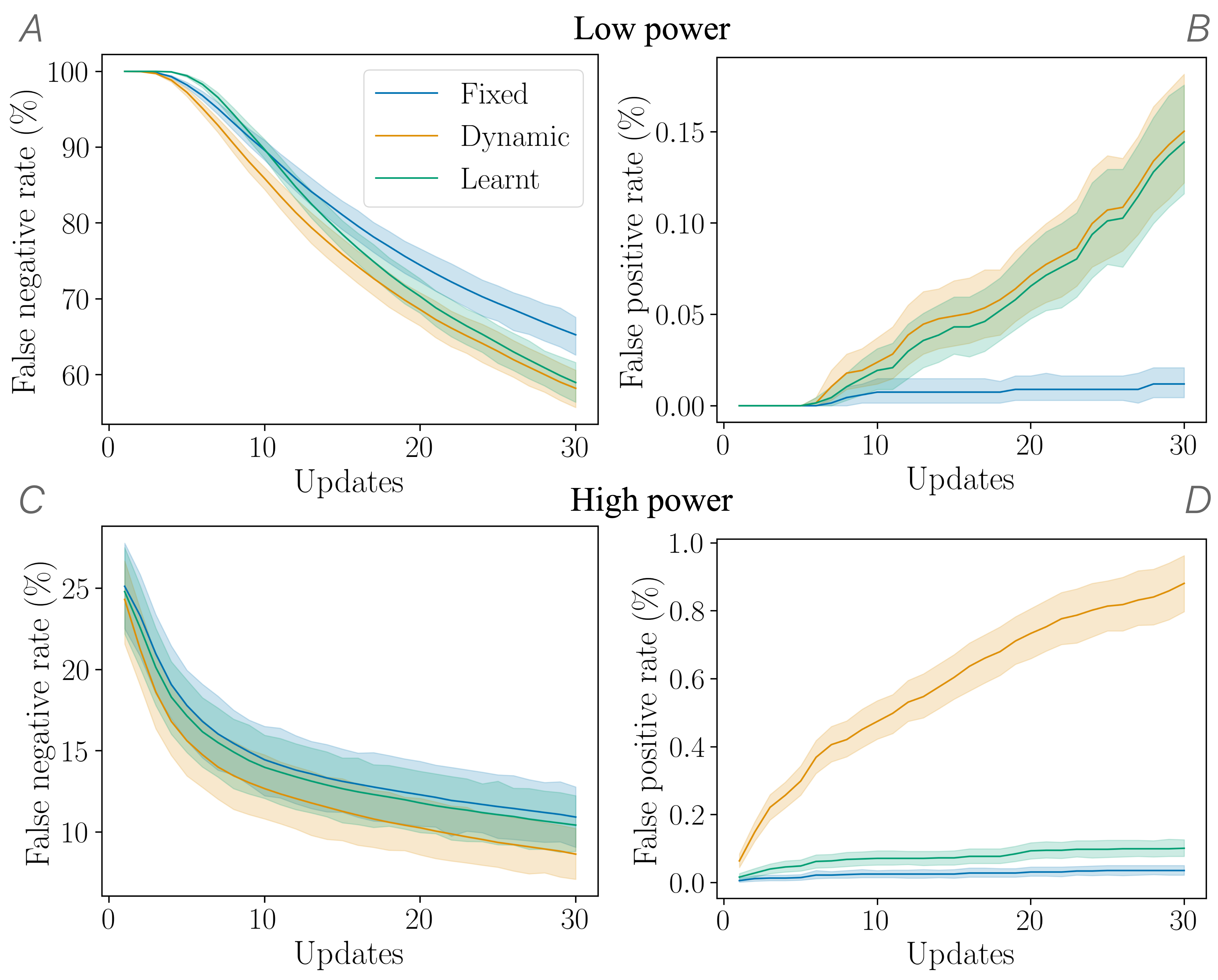}
  \caption{Panels A-D compare false negative and positive rates with the hierarchical Bayesian method, under different simulation scenarios and settings of the $\tau$ parameter.}
  \label{tau_comparison}
\end{figure}

Figs. \ref{tau_comparison}A-D show the evolution of false negative rates and false positive rates with the hierarchical Bayesian method, in the high power and low power scenarios, for the three different approaches to the specification of $\tau$ within the calculation of the Bayes factors. In the low power scenario, the dynamic and learnt specifications have similar false negative rates and false positives rates. Fixing $\tau$ to 0.1, however, shows a noticeably worse false negative rate as the number of updates increase. In the high power scenario, the salient difference instead occurs within dynamic specification of $\tau$. Here, the dynamic specification leads to an improved false negative rate at the cost of a magnitude higher false positive rate. In contrast, fixing $\tau$ leads to a reduced false positive rate, with slight cost to the false negative rate.

\subsection{Real-world results}

\begin{figure}[hbt!]
  \centering
  \includegraphics[width=\linewidth]{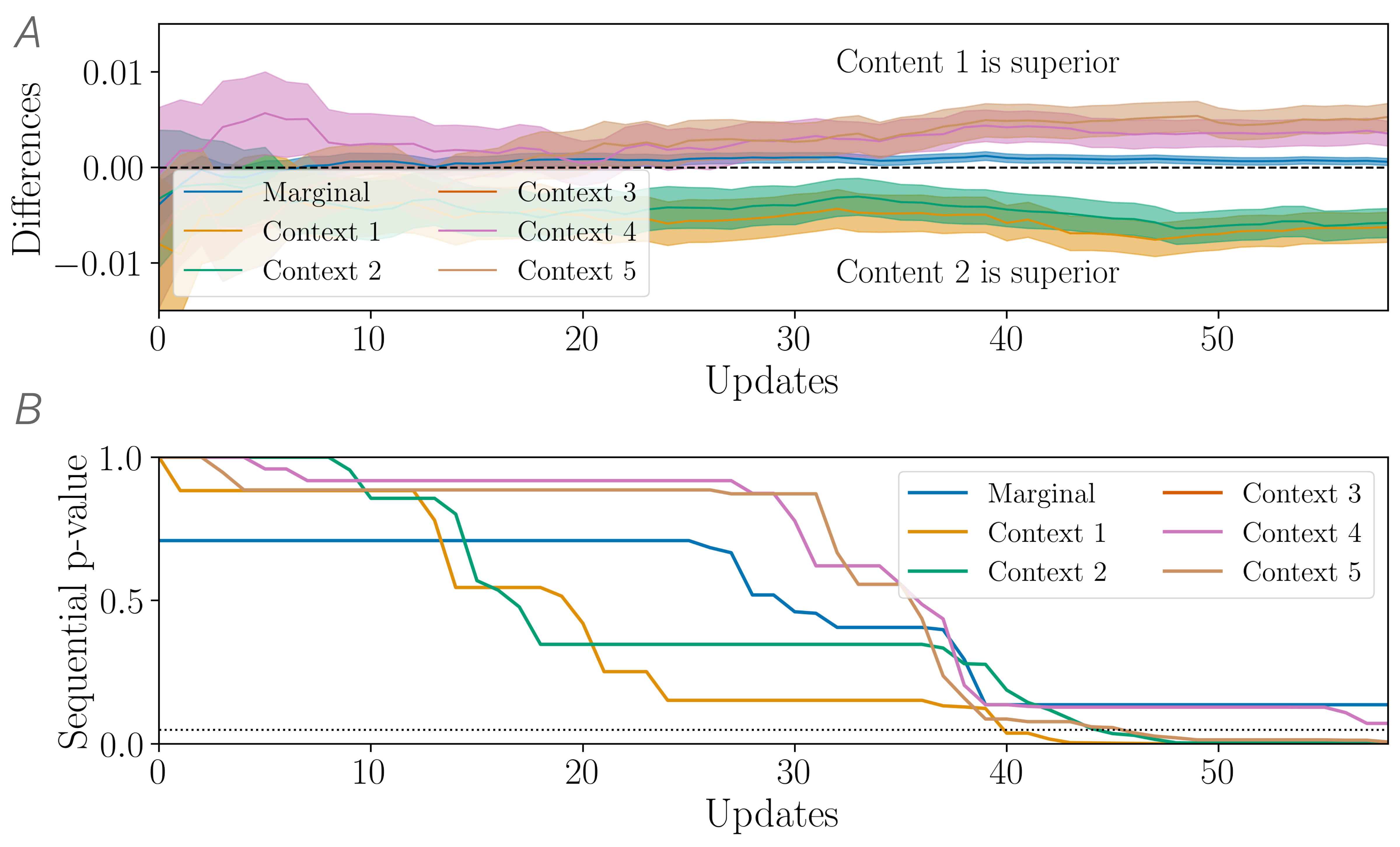}
  \caption{Pair-wise differences (panel A) and sequential p-values (panel B) in an example experiment from the real-world data set. Results are presented within each context, and after marginalising over contexts.}
  \label{real_world_example}
\end{figure}

Figs. \ref{real_world_example}A and \ref{real_world_example}B show results from applying the methodology we have developed to an example experiment in the real-world data set. As with the simulations, $\tau$ was set to $0.1$. In this example experiment, there were five possible values of one contextual factor and two values of one content factor, with $\hat{r}_f$s estimated over 60 sequential updates. Fig. \ref{real_world_example}A shows the evolution of pair-wise differences within each context for this example; for two context values one of the content values appears superior, whereas for the other three the relationship is swapped. Alongside this, the figure shows the pair-wise differences at the global level, i.e., as computed using marginal probability estimation (see Sec. \ref{marginal_probability_estimation}). This marginal difference between the two content values sits in the middle of the five context values. As in fig. \ref{real_world_example}B, the weight of evidence, as quantified by the sequential p-values, suggests that we can trust these conclusions for four of the context factors where the p-values dropped below 0.05. There is insufficient data for the global result to yield a significant p-value. This result demonstrates the capacity of the overall Bayesian framework to distinguish valuable insights within context factors, even when the global behaviour is more muddied.

Across the data set of real-world experiments, we examined the sequential p-values generated by comparing every combination of content values for each context value, using both the maximum likelihood and hierarchical Bayesian methods. These p-values are plotted as a function of the effect size measured by maximum likelihood estimation, in Fig. \ref{real_world_results}A. The figure demonstrates that sequential p-values generated with the hierarchical Bayesian method show a stronger sensitivity to the measured effect size, as evidenced by the faster reduction with increasing effect size. This pattern mirrors the stronger distinction between hierarchical Bayesian p-values observed in $\text{H}_1$ true vs. $\text{H}_0$ true simulations (compare Figs. \ref{statistical_confidence}C and \ref{statistical_confidence}D), when compared to maximum likelihood estimation (compare Figs. \ref{statistical_confidence}A and \ref{statistical_confidence}B). In the real world setting, this feature of the hierarchical Bayesian method reduces false positives under conditions of high noise and weak effects, without the need for post hoc multiple comparisons.

\subsubsection{Meta effect-size prior learning}
Finally, we also applied the meta-prior learning described in Sec. \ref{effect_size_prior_learning} and Fig. \ref{tau_comparison} to the real-world data set, by randomly splitting the data set into equal halves, learning the $\tau$ parameter from the first half, and using this learnt value to estimate sequential p-values in the second half. Fig. \ref{real_world_results}B show the number of experiments with statistically significant detected effects in the second half, for both the fixed and learnt estimates of $\tau$. The benefit of meta-prior learning visualised in Fig. \ref{tau_comparison} is again highlighted here: the proportion of experiments with a detected effect increases by 32\%. In conjunction with the evidence from the numerical simulations, we can posit that the majority of these additional effects detected using a learnt $\tau$ would have otherwise been missed. It's also clear that a dynamic value of $\tau$ \cite{Zhao2018SafelyDesign} achieves much higher detection rates, incurring a higher risk of false positives. This too aligns with findings from simulated data in Fig. \ref{tau_comparison}. In Sec. \ref{Discussion}, we discuss the relative pros and cons of learning vs. setting $\tau$ dynamically in practice.

\begin{figure}[hbt!]
  \centering
  \includegraphics[width=\linewidth]{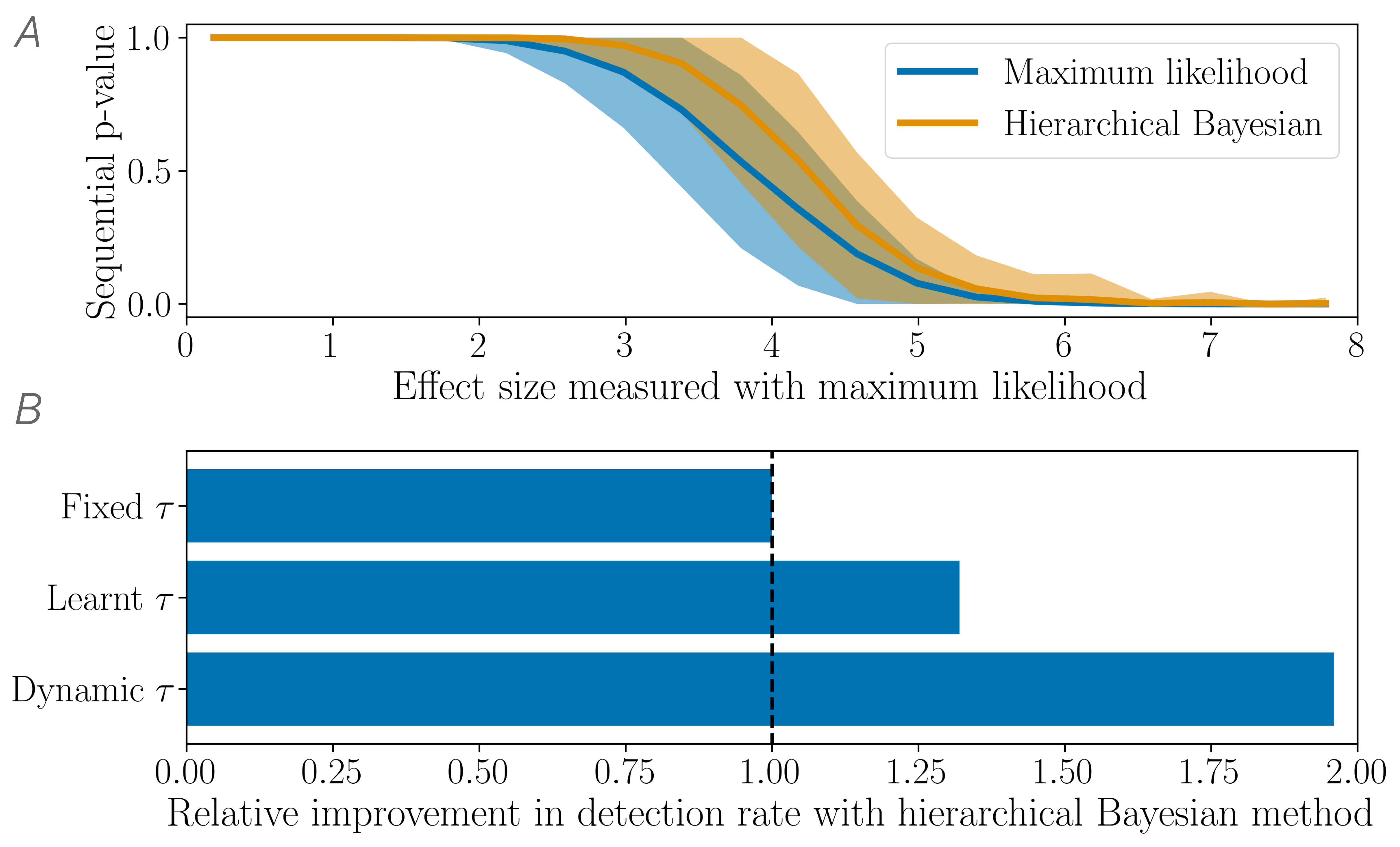}
  \caption{Panel A plots sequential p-values estimated by the hierarchical Bayesian and MLE methods, averaged over all experiments in the real-world data set, as a function of MLE effect size. Panel B plots detection performance with fixed $\tau$ vs. that learnt by building meta-priors.}
  \label{real_world_results}
\end{figure}

\section{Discussion}
\label{Discussion}

The hierarchical Bayesian approach we have described here applies well-established ideas in Bayesian inference to multivariate AB testing. We have combined a theoretical framework, numerical simulations and real-world tests to demonstrate its value for large-scale experimentation in the digital services industry. The underpinning learning methodology is now increasingly applicable, due to the availability of relatively inexpensive computational power combined with efficient implementations.

Comparing this approach with maximum likelihood estimation illustrates why it is useful, and when. As demonstrated, there are scenarios in which it yields a more favourable trade-off between false positive and negative rates It does so without the need for post hoc correction for multiple comparisons, which can be prohibitively expensive in large-scale experiments. From practical experience, we observe that our method is particularly beneficial in cases where partial pooling can learn from traffic that is relatively evenly distributed among all cells within the experiment.

But the benefits to hierarchical Bayesian inference extend beyond this quantitative trade-off. One of the salient features of our hierarchical Bayesian approach is that it regularises reward estimates and effect sizes through prior pooling. The consequence is that, for experiments with low data volumes, all treatments will have more similar reward estimates, closer to the common mean of all the estimates. This yields a practical benefit in real-world systems in which clients use both reported effect sizes and p-values to make decisions. Hence, alongside the reduction in ``statistical'' false positives quantified by p-values, pooling also prevents ``human judgement'' false positives. Experimenters, particularly non-experts, are prone to making spurious conclusions when they observe larger differences between treatments, even if unsupported by statistical inference. By incorporating pooling, hierarchical Bayesian inference reduces the occurrence of spuriously large differences that likely arise from sampling noise in low signal-to-noise conditions. In summary, hierarchical Bayesian inference reduces the composite statistical and human false positive rate in real-world applications.

A second salient feature of hierarchical Bayesian inference is extensibility. It is a general learning framework that can be easily applied to different use cases. In our exposition, we have only modelled a single dependent variable. But the same modelling approach can be easily extended to  multiple dependent variables, or to a more complex model of the relationship between contexts, contents and other, potentially continuous co-variates. Further, other popular statistical techniques can easily be applied in conjunction with the approach: as we and other researchers have shown, hierarchical Bayesian inference can be used to learn effect-size hyperparameters from historical experiments to accelerate future experiments \cite{Deng2013ImprovingData}, and can be combined with automatic optimisation of treatment allocation with multi-armed bandits \cite{Hill2017AnOptimization, Nabi2022BayesianBayes}.

We have shown that learning the $\tau$ hyperparameter yields more benefit than a fixed specification of the parameter. This is true even if only a small number of historical experiments are available, matching the intuition that this kind of parameter value need only be fit with reasonable accuracy \cite{Johari2015AlwaysTesting}. But why would one choose to take this approach when the dynamic setting \cite{Zhao2018SafelyDesign} can yield greater statistical power, with little cost to the false positive rate? Here too, human factors are worth considering. By learning the underlying distribution of effect sizes, we can inform experimenters about the potential value and impact of future experiments, and help educate them as to when it is and is not beneficial to run an experiment. Conditioning such learning on experiment-level features could provide additional value by guiding future experiments to target high-impact interventions.

\bibliographystyle{IEEEtran}
\bibliography{references}

\newpage

\appendix
\section{Appendix} \label{app:A}


\subsection{A simplified hierarchical model} \label{subsec:model}
Recall from Sec. \ref{subsec:multivariate_ab} that the number of unique content and context factors is $M$ and $C$, respectively. We slightly deviate from the notation of the main part of the article and let $m_i$, $i=1,\dots, M$, denote the unique content factors and let $c_j$, $j=1,\dots, C$, be the unique context factors. We have at most $N$ values for each $m_i$ and $c_j$, respectively. Hence, for $F=M+C$ the total number $N_F$ of unique content-context combinations is $O(N^F)$. In addition, let $f=1,\dots,N_F$ denote an enumeration of all unique content-context combinations.

For simplicity we also assume that the observed data consists of a stream of response rates $r_f^{(i)} \in [0,1]$ with $i=1,\dots,n_f$, per content-context combination $f$. We consider these response rates after applying the inverse sigmoid function $g^{-1}$, i.e., we have a data stream of the form 
\begin{equation} \label{eq:data1}
    y_f^{(i)} = g^{-1}\left(r_f^{(i)}\right), \,\,\,f=1,\dots, N_F, \,\,\,i=1,\dots,n_f.
\end{equation}
Following the exposition in \cite[Chapter 5]{gelman2004bayesian}, the random variables $y_f^{(i)}$ are assumed to be independent and normally distributed according to
\begin{equation}\label{eq:likelihood}
    y_f^{(i)}  \sim \mathrm{Normal}\left((X\beta)_f, \sigma_f^2\right),\,\,\text{for all}\,\, f=1,\dots,N_F,
\end{equation}
where the design matrix $X$ is as in \eqref{eq:X} above. We denote by $(X\beta)_f$ the element with index $f$ of the vector $X\beta$. In order to further simplify the analysis in this section we assume that $X$ is the identity matrix of shape $N_F\times N_F$. This implies that the coefficient vector $\beta$ is of the form
\begin{equation} \nonumber
    \beta = (\beta_1,\dots,\beta_{N_F})^T,
\end{equation}
where $v^T$ denotes the transpose of a vector $v$.
The coefficients $\beta_f$, $f=1,\dots,N_F$, are unknown, while the positive standard deviations $\sigma_f$ are assumed to be known. We consider a Bayesian model for the $\beta_f$ similar to Sec. \ref{subsec:HB_inference} above. Specifically, the $\beta_f$ follow
\begin{equation} \label{eq:beta}
\beta_f|\mu \sim \mathrm{Normal}(\mu, \sigma_{\beta}^2),
\end{equation}
where $\sigma_{\beta}>0$ is known and the $\beta_f$ are conditionally independent given $\mu$. In addition, $\mu$ in  \eqref{eq:beta} is normally distributed with
\begin{equation} \label{eq:mu}
\mu\sim \mathrm{Normal}(0, \sigma_{\mu}^2),
\end{equation}
where $\sigma_{\mu}$ is the known positive standard deviation.

Note that \eqref{eq:likelihood} implies that the empirical means $\bar{y}_f^{(\cdot)}:=\frac{1}{n_f}\sum_{i=1}^{n_f} y_f^{(i)}$ are independent and distributed according to
\begin{equation} \label{eq:data3}
    \bar{y}_f^{(\cdot)} \sim \mathrm{Normal}\left(\beta_f, s_f^2\right),\,\,\,\text{with}\,\,\,s_f^2 := \frac{\sigma_f^2}{n_f}.
\end{equation}

\subsection{Explicit hierarchical Bayesian inference}
The following proposition derives the posterior distribution of the parameters $\beta_f$ in closed form. 
For this let $\mathcal{D}:=\left\{\bar{y}_f^{(\cdot)}:\, f=1,\dots,N_F\right\}$ for the sufficient statistics $\bar{y}_f^{(\cdot)}$, $f=1,\dots,N_F$.

\begin{proposition} \label{thm:1}
Given the model in eqns. \eqref{eq:likelihood} -- \eqref{eq:data3}, the posterior of $\beta_f$ is of the form
\begin{equation} \label{eq:post_beta}
    p\left(\beta_f \big| \mathcal{D}\right) = \mathrm{Normal}\left(\hat{\beta}_f, \hat{\sigma}_f^2\right),
\end{equation}
where the mean $\hat{\beta}_f$ is 
\begin{align} \label{eq:beta_hat}
    \hat{\beta}_f &= \frac{\bar{y}_f^{(\cdot)}}{1 + \frac{s_f^2}{\sigma_{\beta}^2}} \nonumber \\
    &+ \frac{1}{1 + \frac{\sigma_{\beta}^2}{s_f^2}} \left( \sum_{j=1}^{N_F} \frac{\big(\sigma_{\beta}^2 + s_j^2\big)^{-1}}{\sigma_{\mu}^{-2} + \sum_{k=1}^{N_F} \big(\sigma_{\beta}^2 + s_k^2\big)^{-1}} \bar{y}_j^{(\cdot)}\right),
\end{align}
and the variance $\hat{\sigma}_f^2$ is of the form
\begin{equation} \label{eq:post_var_beta}\nonumber
    \hat{\sigma}_f^2 = \frac{1}{\frac{1}{\sigma_{\beta}^2} + \frac{1}{s_f^2}} + \left(1 + \frac{\sigma_{\beta}^2}{s_f^2} \right)^{-2} \frac{1}{\sigma_{\mu}^{-2}+\sum_{k=1}^{N_F} \big(s_k^2 + \sigma_{\beta}^2\big)^{-1}}.
\end{equation}
\end{proposition}

\begin{proof}
In order to show that the posterior of $\beta_f$ is as in \eqref{eq:post_beta} we need to carry out an explicit Bayesian analysis for the hierarchical model defined by \eqref{eq:likelihood} -- \eqref{eq:data3}. For this we have from \cite[p.~116]{gelman2004bayesian}
\begin{equation} \nonumber
    p\left(\beta_f \big| \mu, \mathcal{D}\right) = \mathrm{Normal}\left(\tilde{\beta}_f, \tilde{\sigma}_f^2 \right), \,\,\,
\end{equation}
with
\begin{equation} \nonumber
    \tilde{\beta}_f = \frac{s_f^{-2}\bar{y}_f^{(\cdot)}  + \sigma_{\beta}^{-2} \mu}{\sigma_{\beta}^{-2} + s_f^{-2}},
\end{equation}
and
\begin{equation} \nonumber
    \tilde{\sigma}_f^2 = \frac{1}{\sigma_{\beta}^{-2} + s_f^{-2}}.
\end{equation}
Furthermore, we can derive the posterior of $\mu$ as
\begin{equation} \nonumber
    p\left(\mu | \mathcal{D}\right) = \mathrm{Normal}(\tilde{\mu}, \tilde{\sigma}^2),
\end{equation}
where
\begin{equation} \nonumber
    \tilde{\mu} = \sum_{j=1}^{N_F}\frac{\big(s_j^2 + \sigma_{\beta}^2\big)^{-1}}{\sigma_{\mu}^{-2}+ \sum_{k=1}^{N_F} \big(s_k^2 + \sigma_{\beta}^2\big)^{-1}}\bar{y}_{j}^{(\cdot)},
\end{equation}
and
\begin{equation} \nonumber
    \tilde{\sigma}^2 = \frac{1}{\sigma_{\mu}^{-2} + \sum_{k=1}^{N_F} \big(s_k^2+\sigma_{\beta}^2\big)^{-1}}.
\end{equation}
Hence, the posterior of $\beta_f$ can be computed by marginalising the effect of $\mu$, i.e.,
\begin{equation}\label{eq:proof_1_post_int}
    p\left(\beta_f \big| \mathcal{D}\right) = \int_{-\infty}^{\infty} p\left(\beta_f \big| \mu, \mathcal{D}\right)p\left(\mu | \mathcal{D}\right) \mathrm{d} \mu. 
\end{equation}
As both distributions under the integral in \eqref{eq:proof_1_post_int} are Gaussian, we can use standard conjugacy arguments to obtain the result.

\end{proof}

\begin{remark}
The expression for the posterior mean $\hat{\beta}_f$ in \eqref{eq:beta_hat} can be viewed as trading off a single empirical mean $\bar{y}_f^{(\cdot)}$ with a weighted average of the empirical means over all content-context combinations. Balancing these two terms is achieved by comparing our prior uncertainty in $\beta_f$ expressed by $\sigma_{\beta}^2$ in \eqref{eq:beta} with our uncertainty $s_f^2$ in the data for content-context combination $f$. 
\end{remark}

\subsection{Bias and variance analysis}
The mean of the posterior $\hat{\beta}_f$ given in \eqref{eq:beta_hat} can be interpreted as an estimator for the unknown $\beta_f$. In this case the data $\bar{y}_f^{(\cdot)}$, $f=1,\dots,N_F$, on the right-hand side (RHS) of \eqref{eq:beta_hat} are viewed as random variables distributed according to \eqref{eq:data3}.
In the following theorem we provide an expression for the mean and an upper bound for the variance of $\hat{\beta}_f$. Note that in this subsection mean and variance are computed w.r.t. the distribution of $\bar{y}_f^{(\cdot)}$, $f=1,\dots,N_F$, and under the assumption that the true $\beta_f$ are fixed, but unknown.

\begin{proposition} \label{thm:2}
The mean of $\hat{\beta}_f$ given in \eqref{eq:beta_hat} is of the form
\begin{align} \label{eq:expected_beta} 
    \mathbb{E}\left[\hat{\beta}_f\right] &= \frac{\beta_f}{1 + \frac{s_f^2}{\sigma_{\beta}^2}} \nonumber \\
    &+ \frac{1}{1 + \frac{\sigma_{\beta}^2}{s_f^2}} \left( \sum_{j=1}^{N_F} \frac{\big(\sigma_{\beta}^2 + s_j^2\big)^{-1}}{\sigma_{\mu}^{-2} + \sum_{k=1}^{N_F} \big( \sigma_{\beta}^2 + s_k^2\big)^{-1}} \beta_j\right),
\end{align}
and the variance of $\hat{\beta}_f$ admits the following upper bound
\begin{align} \label{ineq:var_beta_hat}
    &\mathrm{Var}\left(\hat{\beta}_f\right) \leq \frac{s_f^2}{\left(1 + \frac{s_f^2}{\sigma_{\beta}^2}\right)^2} + \frac{2 \sigma_{\beta}^2}{\left(1 + \frac{\sigma_{\beta}^2}{s_f^2}\right)^2} \nonumber \\
    &+ \frac{1}{\left(1 + \frac{\sigma_{\beta}^2}{s_f^2}\right)^2}\sum_{j=1}^{N_F} \left(\frac{\big(\sigma_{\beta}^2 + s_j^2\big)^{-1}}{\sigma_{\mu}^{-2} + \sum_{k=1}^{N_F} \big(\sigma_{\beta}^2 + s_k^2\big)^{-1}}\right)^2 s_j^2.
\end{align}
\end{proposition}

\begin{proof}
In order to establish \eqref{eq:expected_beta} recall that the empirical means $\bar{y}_f^{(\cdot)}$ are distributed as in \eqref{eq:data3}. Hence, taking expectations on both sides of \eqref{eq:beta_hat} yields the result.

For bounding the variance of $\hat{\beta}_f$ we use the assumption that the $\bar{y}_f^{(\cdot)}$ are independent. Hence, inequality \eqref{ineq:var_beta_hat} follows after grouping all terms in \eqref{eq:beta_hat} by the respective $f$, followed by taking the variance on both sides.
\end{proof}

\begin{remark} \label{remark:beta_is_biased}
Proposition \ref{thm:2} provides us with a closed form expression for the expected value of $\hat{\beta}_f$. Note that $\hat{\beta}_f$ is not necessarily an unbiased estimator of $\beta_f$. 
\end{remark}

Recall that the maximum likelihood estimator for the content-context combination $f$ within the above model is given by the empirical mean $\bar{y}_f^{(\cdot)}$, which is an unbiased estimator for the true mean $\beta_f$ and has variance $s_f^2$ as defined in \eqref{eq:data3}. The following theorem provides sufficient conditions under which the Bayesian estimator $\hat{\beta}_f$ achieves lower variance than the corresponding maximum likelihood estimator for a content-context combination $f$.

\begin{theorem}\label{thm:3}
Let $\sigma_{\beta}^2$, $\sigma_{\mu}^2$ and $s_f^2$, $f=1,\dots,N_F$, be defined as in \eqref{eq:beta}, \eqref{eq:mu} and \eqref{eq:data3}, respectively. Furthermore, set $h := \frac{s_f^2}{\sigma_{\beta}^2}$ and let $c>0$ be such that $\frac{\sigma_{\mu}^2}{\sigma_{\beta}^2} \leq c$.
Then, the following two statements hold.
\begin{enumerate}
    \item If 
    \begin{equation} \label{thm1:cond_1}
    \max_{j=1,\dots, N_F;\,\, j\neq f} \frac{s_j^2}{\sigma_{\beta}^2} \leq \frac{1}{h},
    \end{equation}
    then
    \begin{equation} \label{ineq:thm1_main_1}
        \mathrm{Var}\left(\hat{\beta}_f\right) \leq c_1(h) s_f^2,
    \end{equation}
    where
    \begin{align} 
        c_1(h) &= \frac{1}{(1+h)^2} + \frac{2}{h\left(1 + \frac{1}{h}\right)^2} \nonumber \\
        &+ \frac{c^2}{\left( 1 + h\right)^2 \left(1 + \frac{1}{h}\right)^2 }+\frac{1}{h^2 \left(1 + \frac{1}{h}\right)^2}. \label{eq:thm1_c1}
    \end{align}
    \item If $\frac{s_j^2}{\sigma_{\beta}^2} = h$, for all $j=1,\dots,N_F$, then
    \begin{equation}
        \mathrm{Var}\left(\hat{\beta}_f\right) \leq c_2(h) s_f^2,
    \end{equation}
    where
    \begin{equation}\nonumber
       c_2(h) = \frac{1}{\left(1 + h\right)^2} + \frac{2}{h\left(1 + \frac{1}{h}\right)^2} + \frac{1}{\left(1 + \frac{1}{h}\right)^2}.
    \end{equation}
    \end{enumerate}
\end{theorem}

\begin{proof}
To prove part (1) we bound each term on the RHS of inequality \eqref{ineq:var_beta_hat} separately. In this respect, we have that 
\begin{equation} \label{eq:thm1_proof_1}
    \frac{s_f^2}{\left(1 + \frac{s_f^2}{\sigma_{\beta}^2}\right)^2} = \frac{1}{(1+h)^2}s_f^2.
\end{equation}
Furthermore,
\begin{equation} \label{eq:thm1_proof_2}
    \frac{2 \sigma_{\beta}^2}{\left(1 + \frac{\sigma_{\beta}^2}{s_f^2}\right)^2} = \frac{2}{h\left(1+\frac{1}{h}\right)^2}s_f^2,
\end{equation}
where we used that $s_f^2 = h\sigma_{\beta}^2$ by definition of $h$.

Finally, we consider the last term on the RHS of inequality \eqref{ineq:var_beta_hat}. For this note that the term for $j=f$ under the sum can be estimated as follows:
\begin{align}
   &\frac{1}{\left(1 + \frac{\sigma_{\beta}^2}{s_f^2}\right)^2} \left(\frac{\big(\sigma_{\beta}^2 + s_f^2\big)^{-1}}{\sigma_{\mu}^{-2} + \sum_{k=1}^{N_F} \big(\sigma_{\beta}^2 + s_k^2\big)^{-1}}\right)^2  s_f^2 \nonumber \\
   &\leq
   \frac{\sigma_{\mu}^4}{\left(1+\frac{1}{h}\right)^2} \frac{s_f^2}{\left(\sigma_{\beta}^2 + s_f^2\right)^2} \nonumber \\
   &= \frac{\sigma_{\mu}^4}{\sigma_{\beta}^4} \frac{1}{(1+h)^2 \left(1 + \frac{1}{h}\right)^2}s_f^2, \label{eq:thm1_proof_3}
\end{align}
where we used the definition of $h$. In addition, note that 
\begin{equation} \nonumber
    \sum_{j=1}^{N_F} \left(\frac{\big(\sigma_{\beta}^2 + s_j^2\big)^{-1}}{\sigma_{\mu}^{-2} + \sum_{k=1}^{N_F} \big(\sigma_{\beta}^2 + s_k^2\big)^{-1}}\right)^2 \leq 1,
\end{equation}
which implies that 
\begin{align}
    &\frac{1}{\left(1 + \frac{\sigma_{\beta}^2}{s_f^2}\right)^2}\sum_{j\neq f} \left(\frac{\big(\sigma_{\beta}^2 + s_j^2\big)^{-1}}{\sigma_{\mu}^{-2} + \sum_{k=1}^{N_F} \big(\sigma_{\beta}^2 + s_k^2\big)^{-1}}\right)^2 s_j^2 \nonumber \\
    &\leq 
    \frac{\sigma_{\beta}^2}{h\left(1 + \frac{1}{h}\right)^2}
    =\frac{1}{h^2\left(1 + \frac{1}{h}\right)^2} s_f^2.\label{eq:thm1_proof_4}
\end{align}
Combining \eqref{eq:thm1_proof_1} -- \eqref{eq:thm1_proof_4} yields the result.

Part (2) of the theorem follows after substituting $\frac{s_j^2}{\sigma_{\beta}^2}$, $j=1,\dots,N_F$, with $h$ in inequality \eqref{ineq:var_beta_hat}.
\end{proof}
\begin{remark}
 Theorem \ref{thm:3} introduces a variable $h$ expressing the uncertainty $s_f^2$ in the data as a fraction of the uncertainty $\sigma_{\beta}^2$ in the prior for $\beta$. This, together with condition \eqref{thm1:cond_1}, impose a gap relative to $\sigma_{\beta}^2$ between $s_f^2$ for content-context combination $f$ and $s_j^2$ for all other content-context combinations.
Additionally, note that for $c_1(h)$ in \eqref{eq:thm1_c1}, we have that $c_1(h) = O\left(\frac{1}{h}\right)$. Hence, inequality \eqref{ineq:thm1_main_1} shows that for large enough $h$ the variance of $\hat{\beta}_f$ shrinks below the variance $s_f^2$ of the maximum likelihood estimator.
\end{remark}

\begin{remark}
Part (2) of Theorem \ref{thm:3} suggests that in the absence of any difference in uncertainties between content-context combinations the Bayesian estimator $\hat{\beta}_f$ may not exhibit lower variance relative to the maximum likelihood estimator as 
\begin{equation} \nonumber
    c_2(h) = O(1), \quad \text{as}\; h\rightarrow \infty.
\end{equation}
\end{remark}

\subsection{Pseudocode for hierarchical Bayesian inference} \label{app:ps_code_BHM}

As outlined in Sec. \ref{subsec:HB_inference} above we perform inference on the Bayesian hierarchical model in \eqref{eq:beta_full_model} -- \eqref{eq:sigma_full_model} using numpyro and JAX. This requires a functional specification of the hierarchical model allowing us to fit its parameters using MCMC. To enable reproducibilty of our results, this functional pseudocode of the Bayesian hierarchical model is described in Algorithm \ref{alg:BHM_func}. We translated this into Python code in numpyro \cite{phan2019composable}, a popular probabilistic programming language.

\begin{algorithm}
\caption{Functional form of the Bayesian hierarchical model}\label{alg:BHM_func}
\begin{algorithmic}
\Require Design matrix $X$, mean $m$ and standard deviation $s$ for prior $\mu$, scale parameter $b$ for the Half-Cauchy prior $\sigma$, observed assignments $a$ and observed responses $r$
\Procedure{model}{$X, a, r$}
\State $\epsilon \sim \mathrm{Normal}(0, 1) $ 
\State $\mu \sim \mathrm{Normal}(m, s^2) $ \Comment{sample prior mean $\mu$}
\State $\sigma \sim \mathrm{HalfCauchy}(b) $ \Comment{sample prior $\sigma$}
\State $\beta \sim \mathrm{Normal}(\mu, \sigma^2) $ \Comment{sample coefficients $\beta$}
\State $\hat{r}^{h}_f \gets g(X\beta + \epsilon) $  \Comment{response probabilities}
\State $r = \mathrm{Binomial}(a, \hat{r}^{h}_f)$ \Comment{modelled vs. observed responses}
\EndProcedure
\end{algorithmic}
\end{algorithm}
Note that for the Bayesian hierarchical model in \eqref{eq:beta_full_model} -- \eqref{eq:sigma_full_model} of Sec.~\ref{subsec:HB_inference} we have that $m=0$, $s^2 = 100$ and $b=5$.

After specifying a functional form for the Bayesian hierarchical model we proceed to infer its parameters. Specifically, we estimate the parameters $\beta, \mu, \sigma$ and $\hat{r}^{h}_f$ using the No-U-Turn Sampler (NUTS) \cite{Hoffman2014TheCarlo} which is an extension of Hamiltonian Monte Carlo. We also refer readers to \cite{Betancourt2015HamiltonianModels}, which addresses the application of Hamiltonian Monte Carlo methods specifically to hierarchical models.

\begin{algorithm}
\caption{Fit the Bayesian hierarchical model using MCMC}\label{alg:BHM_func_fit}
\begin{algorithmic}
\Require Design matrix $X$, observed assignments $a$ and observed responses $r$
\Procedure{fit}{$X, a, r$}
\State $\mathrm{model} \gets \mathrm{NUTS}(\mathrm{MODEL}).\mathrm{fit}(X, a, r)$
\EndProcedure
\end{algorithmic}
\end{algorithm}

The model fitted in Algorithm \ref{alg:BHM_func_fit} can be used to generate samples from the posterior distribution of $\hat{r}^{h}_f$ for each factor vector $f$. These samples are then used to approximate Bayes factors in Sec. \ref{statistical_hypothesis_testing} above to perform statistical hypothesis testing. 

\balance

\end{document}